\documentclass[runningheads,a4paper]{llncs}

\usepackage{amssymb}
\usepackage{graphicx}
\usepackage{caption}
\usepackage{subcaption}

\usepackage{algorithm}
\usepackage{algorithmic}

\usepackage{hyperref}
\usepackage{amsmath}
\usepackage{amsfonts}
\usepackage{bbm}
 
\usepackage{tikz}

\usepackage{url}
\urldef{\mailsa}\path|{zlipton, celkan, muralib}|
\urldef{\mailsc}\path|@cs.ucsd.edu|    
\newcommand{\keywords}[1]{\par\addvspace\baselineskip
\noindent\keywordname\enspace\ignorespaces#1}

\begin{document}

\mainmatter  

\title{Thresholding Classifiers to Maximize F1 Score}


%
%
\author{Zachary C. Lipton \and Charles Elkan \and Balakrishnan Naryanaswamy}
%

\institute{University of California, San Diego,\\
La Jolla, California, 92093-0404, USA\\
\mailsa \mailsc\\    }
%
\maketitle

\begin{abstract}
This paper provides new insight into maximizing F1 scores 
in the context of binary classification
and also in the context of multilabel classification. 
The harmonic mean of precision and recall, F1 score is widely used 
to measure the success of a binary classifier when one class is rare. 
Micro average, macro average, and per instance average F1 scores are used in multilabel classification. 
For any classifier that produces a real-valued output, 
we derive the relationship between the best achievable F1 score and 
the decision-making threshold that achieves this optimum. 
As a special case, if the classifier outputs are well-calibrated conditional probabilities, 
then the optimal threshold is half the optimal F1 score. 
As another special case, if the classifier is completely uninformative, 
then the optimal behavior is to classify all examples as positive. 
Since the actual prevalence of positive examples typically is low, 
this behavior can be considered undesirable. 
As a case study, we discuss the results, which can be surprising, 
of applying this procedure when predicting 26,853 labels for Medline documents. 
\end{abstract}

\keywords{machine learning, evaluation methodology, F1-score, multilabel classification, binary classification}

\section{Introduction}

Performance metrics are useful for comparing the quality of predictions across systems. 
Some commonly used metrics for binary classification 
are accuracy, precision, recall, F1 score, and Jaccard index \cite{systematic}.
Multilabel classification is an extension of binary classification
that is currently an area of active research in supervised machine learning \cite{tsoumakas2007multi}. 
Micro averaging, macro averaging, and per instance averaging
are three commonly used variants of F1 score used in the multilabel setting. 
In general, macro averaging increases the impact on final score of performance on rare labels,
while per instance averaging increases the importance of performing well on each example \cite{tanmacro}. 
In this paper, we present theoretical and experimental results on the properties of the F1 metric.%
\footnote{%
For concreteness, the results of this paper are given specifically for the F1 metric
and its multilabel variants.
However, the results can be generalized to F$\beta$ metrics for $\beta \not= 1$.
}

Two approaches exist for optimizing performance on F1. 
Structured loss minimization incorporates the performance metric into the loss function and then optimizes during training. 
In contrast, plug-in rules convert the numerical outputs of a classifier into optimal predictions \cite{plugin}. 
In this paper, we highlight the latter scenario to differentiate between 
the beliefs of a system and the predictions selected to optimize alternative metrics. 
In the multilabel case, we show that the same beliefs can produce markedly dissimilar optimally thresholded predictions depending upon the choice of averaging method. 

That F1 is asymmetric in the positive and negative class is well-known.
Given complemented predictions and actual labels, 
F1 may award a different score.  
It also generally known that micro F1 
is affected less by performance on rare labels, 
while Macro-F1 weighs the F1 of on each label equally \cite{manning2008introduction}. 
In this paper, we show how these properties are manifest 
in the optimal decision-making thresholds 
and introduce a theorem to describe that threshold.
Additionally, we demonstrate that given an uninformative classifier, 
optimal thresholding to maximize F1 predicts all instances positive regardless of the base rate.

While F1 is widely used, some of its properties are not widely recognized. 
In particular, when choosing predictions to maximize the expectation of F1 for a batch of examples, 
each prediction depends not only on the probability that the label applies to that example,
but also on the distribution of probabilities \emph{for all other examples in the batch}. 
We quantify this dependence in Theorem~\ref{th:maxf1}, 
where we derive an expression for optimal thresholds. 
The dependence makes it difficult to relate predictions that are optimally thresholded  for F1
to a system's predicted probabilities.

We show that the difference in F1 score 
between perfect predictions and optimally thresholded random guesses 
depends strongly on the base rate. 
As a result, assuming optimal thresholding and a classifier outputting calibrated probabilities,
predictions on rare labels typically gets a score between close to zero and one,
while scores on common labels will always be high. 
In this sense, macro average F1 can be argued not to weigh labels equally, 
but actually to give greater weight to performance on rare labels.

As a case study, we consider tagging articles in the biomedical literature with MeSH terms, 
a controlled vocabulary of 26,853 labels. 
These labels have heterogeneously distributed base rates. 
We show that if the predictive features for rare labels are lost
(because of feature selection or another cause)
then the optimal threshold to maximize macro F1 leads to predicting these rare labels frequently.
For the case study application, and likely for similar ones, this behavior is far from desirable.

\section{Definitions of Performance Metrics}

\begin{figure}[t]
\[ \begin{array}{l | cc}
\mbox{} & \text{Actual Positive} & \text{Actual Negative} \\
\hline
\mbox{Predicted Positive} & tp & fp \\
\mbox{Predicted Negative} & fn & tn \\
\hline
 \end{array}\] 
\caption{Confusion Matrix}
\label{fig:confusion}
\end{figure}

Consider binary classification in the single or multilabel setting. 
Given training data of the form
$\{ \langle \boldsymbol {x_{1}},\boldsymbol{ y_{1}} \rangle, \hdots, 
\langle \boldsymbol {x_{n}}, \boldsymbol{y_{n} }\rangle  \}$
where each $\boldsymbol {x_{i}}$ is a feature vector of dimension $d$ 
and each $\boldsymbol{y_{i}}$ is a binary vector of true labels of dimension $m$, 
a probabilistic classifier outputs a model which specifies 
the conditional probability of each label applying to each instance given the feature vector. 
For a batch of data of dimension $n \times d$, 
the model outputs an $n \times m$ matrix $C$ of probabilities. 
In the single-label setting, $m = 1$ and $C$ is an $n \times 1$ matrix, i.e.~a column vector.

A decision rule $D(C): \mathbb{R}^{n \times m} \to \{0,1\}^{n \times m}$ 
converts a matrix of probabilities $C$ to binary predictions $P$. 
The gold standard $G \in \mathbb{R}^{n \times m}$ represents 
the true values of all labels for all instances in a given batch.  
A performance metric $M$ assigns a score to a prediction given a gold standard:
\[ 
M(P | G): \{0,1\}^{n \times m} \times \{0, 1\}^{n \times m} \to \mathbb{R} \in [0, 1].
\]
The counts of true positives $tp$, false positives $fp$, false negatives $fn$, 
and true negatives $tn$ are represented via a confusion matrix (Figure~\ref{fig:confusion}).

Precision $p = tp/(tp+fp)$ is the fraction of all positive predictions that are true positives,
while recall $r = tp / (tp + fn)$ is the fraction of all actual positives that are predicted positive.
By definition the F1 score is the harmonic mean of precision and recall:
$F1 = 2/(1/r + 1/p)$. 
By substitution, F1 can be expressed as a function of counts of true positives, false positives and false negatives:
\begin{equation}
	F1 = \frac{2tp}{2tp + fp + fn}.
	\label{eq:f1def}
\end{equation}
The harmonic mean expression for F1 is undefined when $tp=0$, 
but the translated expression is defined. 
This difference does not impact the results below.

\begin{figure}[t]
\centering
\begin{minipage}{.45\textwidth}
	\centering
         \includegraphics[width=2.5in]{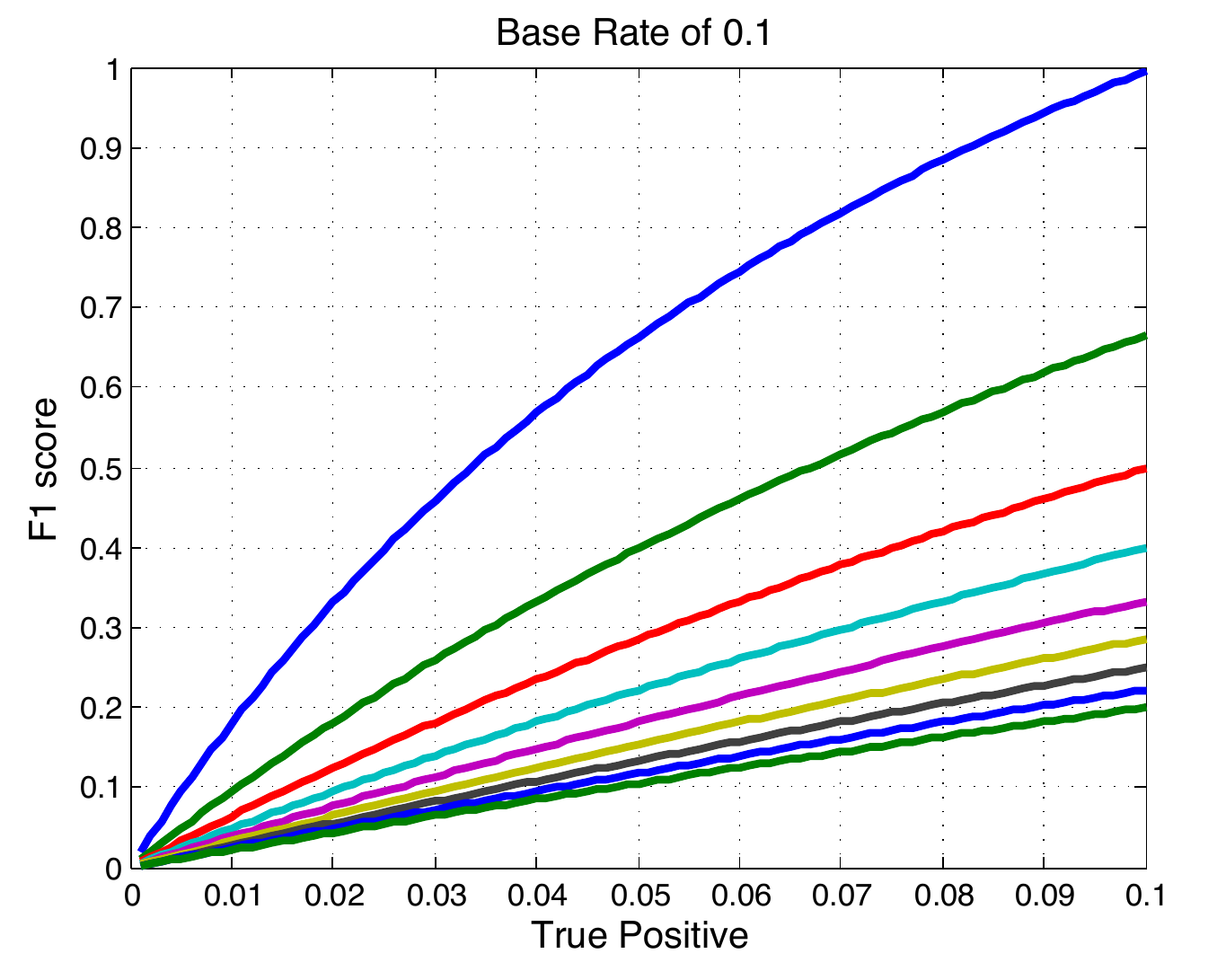}
	\captionof{figure}{Holding base rate and $fp$ constant, F1 is concave in $tp$. 
	Each line is a different value of $fp$.}
	\label{fig:tpconcave}
\end{minipage}%
\hspace{10.0mm}
\begin{minipage}{.45\textwidth}
	\centering
         \includegraphics[width=2.5in]{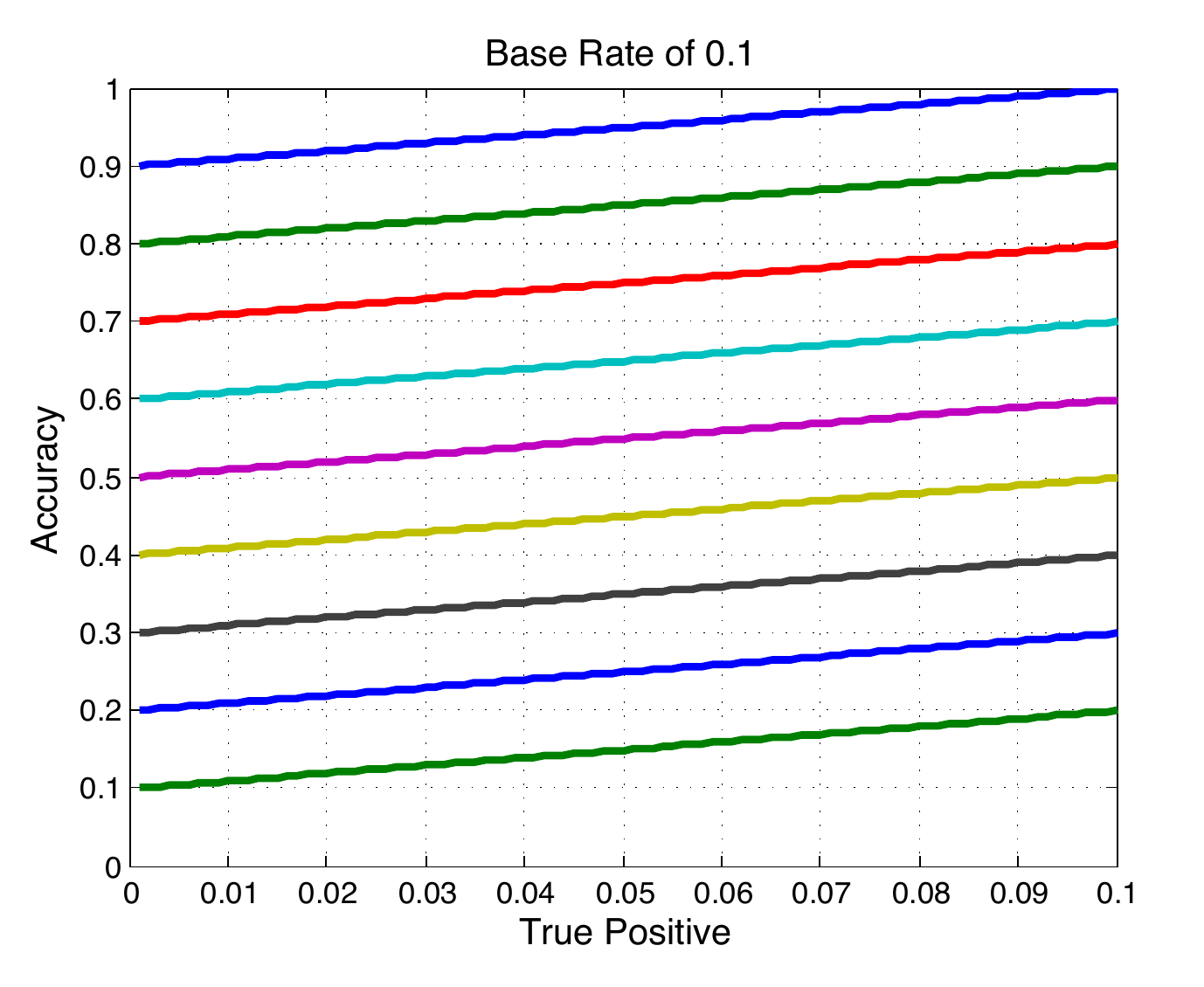}
	\captionof{figure}{Unlike F1, accuracy offers linearly increasing returns. 
	Each line is a fixed value of $fp$.}
	\label{fig:acclinear}
\end{minipage}
\end{figure}
\begin{figure}[t]
\begin{center}
	\includegraphics[width=2.5in]{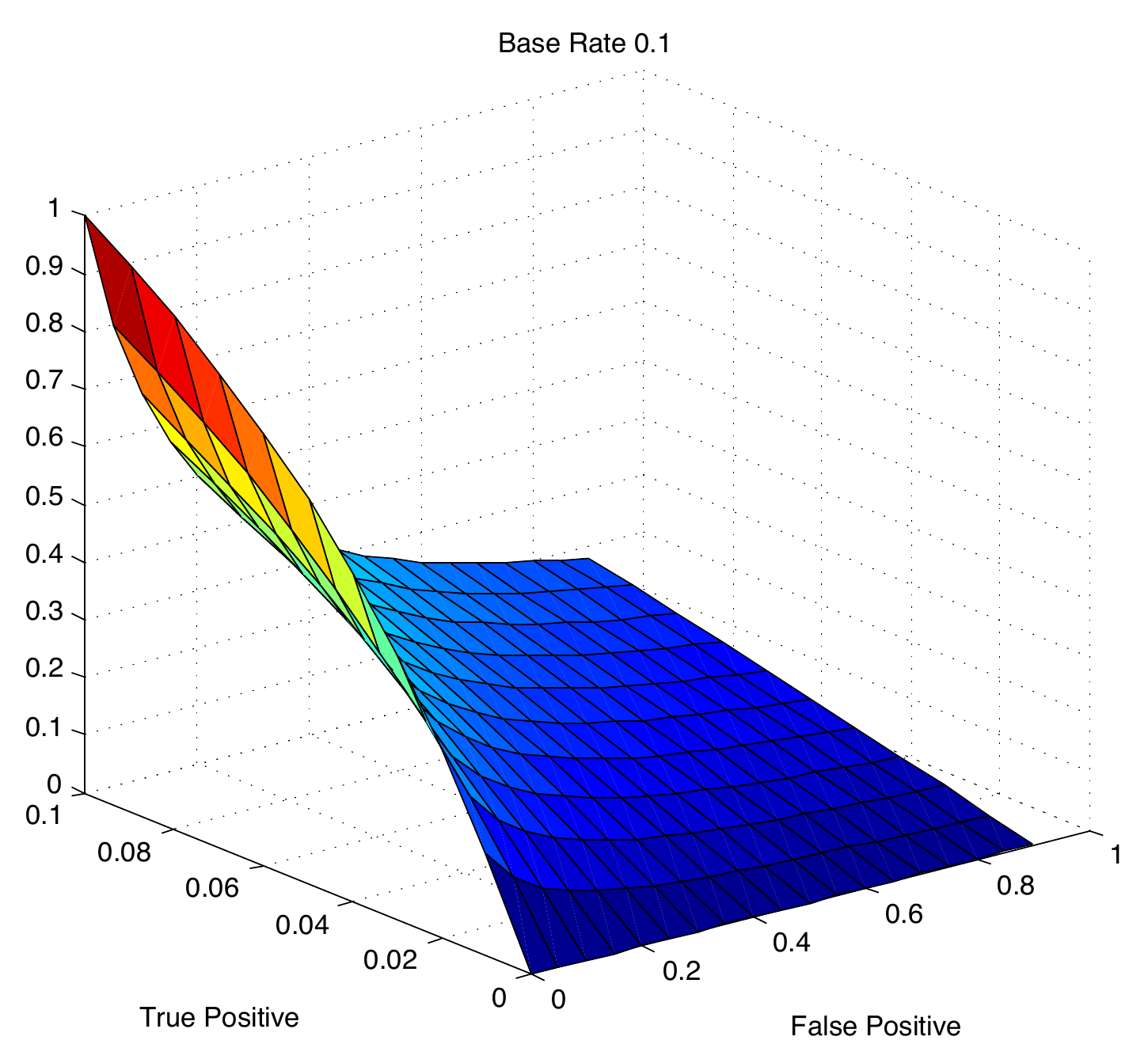}
	\caption{For fixed base rate, F1 is a non-linear function with only two degrees of freedom.}
\end{center}
\end{figure}

\subsection{Basic Properties of F1}

Before explaining optimal thresholding to maximize F1, 
we first discuss some properties of F1. 
For any fixed number of actual positives in the gold standard, 
only two of the four entries in the confusion matrix (Figure~\ref{fig:confusion}) vary independently. 
This is because the number of actual positives is equal to the sum $tp + fn$ 
while the number of actual negatives is equal to the sum $tn + fp$.
A second basic property of F1 is that it is non-linear in its inputs. 
Specifically, fixing the number $fp$, F1 is concave as a function of $tp$ (Figure~\ref{fig:tpconcave}). 
By contrast, accuracy is a linear function of $tp$ and $tn$ (Figure~\ref{fig:acclinear}). 

As mentioned in the introduction, F1 is asymmetric. 
By this, we mean that the score assigned to a prediction $P$ given gold standard $G$ can be arbitrarily different from the score assigned to a complementary prediction $P^{c}$ given complementary gold standard $G^{c}$. 
This can be seen by comparing Figure~\ref{fig:tpconcave} with Figure~\ref{fig:tnconvex}. 
This asymmetry is problematic when both false positives and false negatives are costly. 
For example, F1 has been used to evaluate the classification of tumors as benign or malignant \cite{akay2009support},
a domain where both false positives and false negatives have considerable costs.

\subsection{Multilabel Performance Measures}

While F1 was developed for single-label information retrieval, 
as mentioned there are variants of F1 for the multilabel setting. 
Micro F1 treats all predictions on all labels as one vector 
and then calculates the F1 score.
In particular, 
$$
tp =  2\sum_{i=1}^{n}\sum_{j=1}^{m} \mathbbm{1}(P_{ij}=1)\mathbbm{1}(G_{ij}=1).
$$
We define $fp$ and $fn$ analogously and calculate the final score using (\ref{eq:f1def}). 
Macro F1, which can also be called per label F1,
calculates the F1 for each of the $m$ labels and averages them:
$$
F1_{Macro}(P | G) = \frac{1}{m} {\sum_{j=1}^{m} F1(P_{:j},G_{:j})}.
$$
Per instance F1 is similar but averages F1 over all $n$ examples:
$$
F1_{Instance}(P | G) = \frac{1}{n} {\sum_{i=1}^{n} F1(P_{i:},G_{i:})}.
$$
Accuracy is the fraction of all instances that are predicted correctly:
\begin{equation*}
Acc = \frac{tp+tn}{tp+tn+fp+fn}.
\label{eq:accuracy}
\end{equation*}
Accuracy is adapted to the multilabel setting by summing $tp$ and $tn$ for all labels 
and then dividing by the total number of predictions:
\[
Acc(P|G) = \frac{1}{nm} \sum_{i=1}^{n}\sum_{j=1}^{m}\mathbbm{1}( P_{ij} = G_{ij}).
\]
Jaccard Index, a monotonically increasing function of F1, 
is the ratio of the intersection of predictions and gold standard to their union:  
$$
Jaccard = \frac{tp}{tp +fn + fp}.
$$

\section{Prior Work}
Motivated by the widespread use of F1 in information retrieval 
and in single and multilabel binary classification, 
researchers have published extensively on its optimization.
\cite{jansche2007maximum} propose an outer-inner maximization technique for F1 maximization,
and \cite{jose2009learning} study extensions to the multilabel setting, 
showing that simple threshold search strategies are sufficient 
when individual probabilistic classifiers are independent.
Finally, \cite{dembczynski2011exact} describe how 
the method of \cite{jansche2007maximum} can be extended to efficiently label data points 
even when classifier outputs are dependent. 
More recent work in this direction can be found in \cite{ye2012optimizing}. 
However, none of this work directly identifies the relationship of optimal thresholds 
to the maximum achievable F1 score over all thresholds, as we do here. 

While there has been work on applying general constrained optimization techniques to related metrics  \cite{mozer2001prodding}, 
research often focuses on specific classification methods. 
In particular, \cite{suzuki2006training} study F1 optimization for conditional random fields 
and \cite{musicant2003optimizing} perform the same optimization for SVMs. 
In our work, we study the consequences of such optimization for probabilistic classifiers, 
particularly in the multilabel setting.

A result similar to our special case (Corollary 1) was recently derived in \cite{zhao2013beyond}. 
However, their derivation is complex and does not prove our more general Theorem~\ref{th:maxf1} 
which describes the optimal decision-making threshold 
even when the scores output by a classifier are not probabilities. 
Their paper also does not contain the empirical version we derive 
for the multilabel setting in Theorem~\ref{th:f1multi}. 

The batch observation is related to the observation in \cite{lewis1995} that 
given some classifier, a specific example may or may not cross the decision threshold, 
depending on the other examples present in the test data. 
However, they do not identify this threshold as $\frac{F1}{2}$ 
or make use of this fact to explain the differences between 
predictions made to optimize micro and macro average F1.


\begin{figure}[t]
\centering
\begin{minipage}{.45\textwidth}
	\centering
         \includegraphics[width=2.5in]{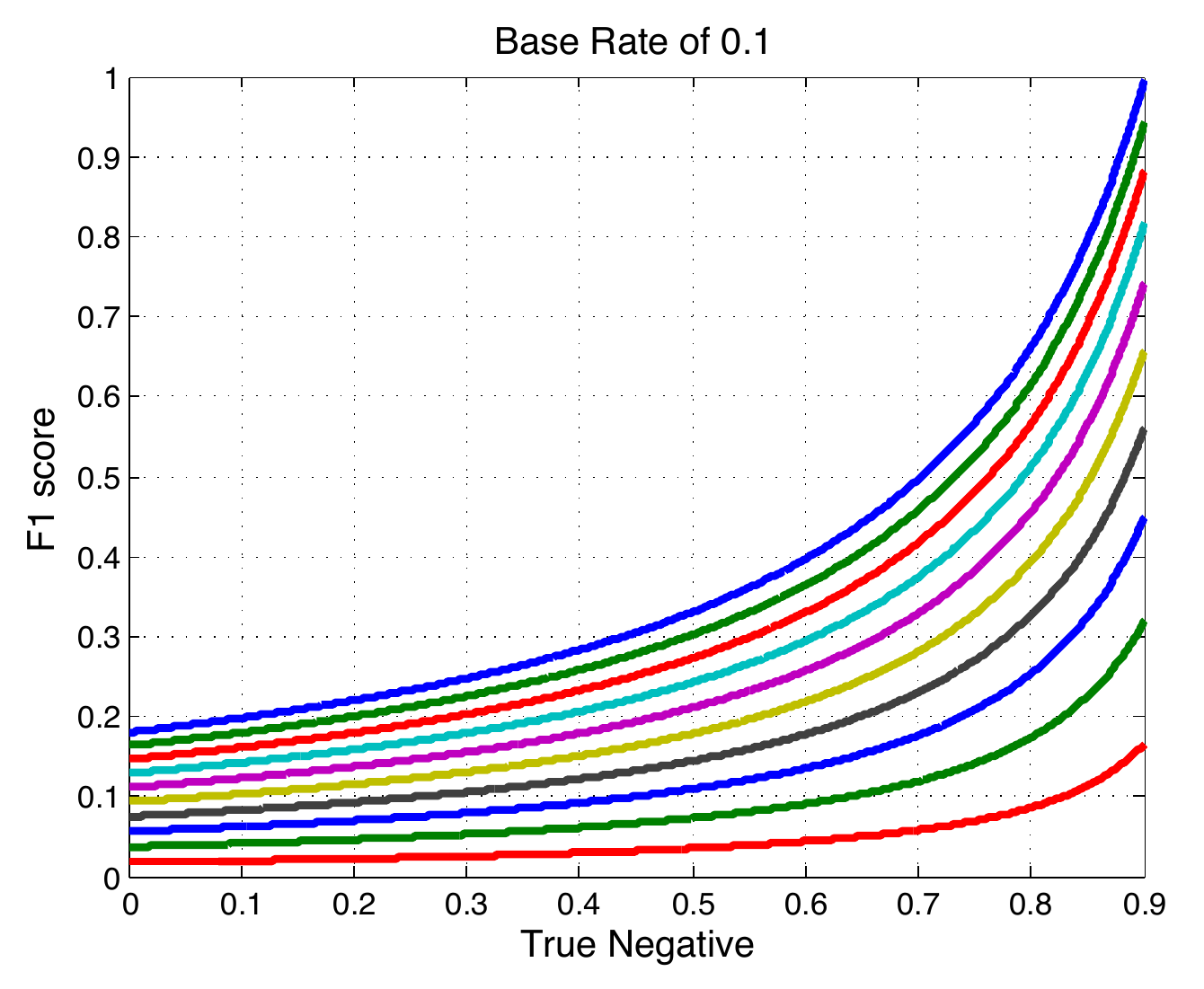}
	\captionof{figure}{F1 score for fixed base rate and number $fn$ of false negatives. 
	F1 offers increasing marginal returns as a function of $tn$. Each line is a fixed value of $fn$.}
	\label{fig:tnconvex}
\end{minipage}%
\hspace{10.0mm}
\begin{minipage}{.45\textwidth}
	\centering
	\includegraphics[width=2.5in]{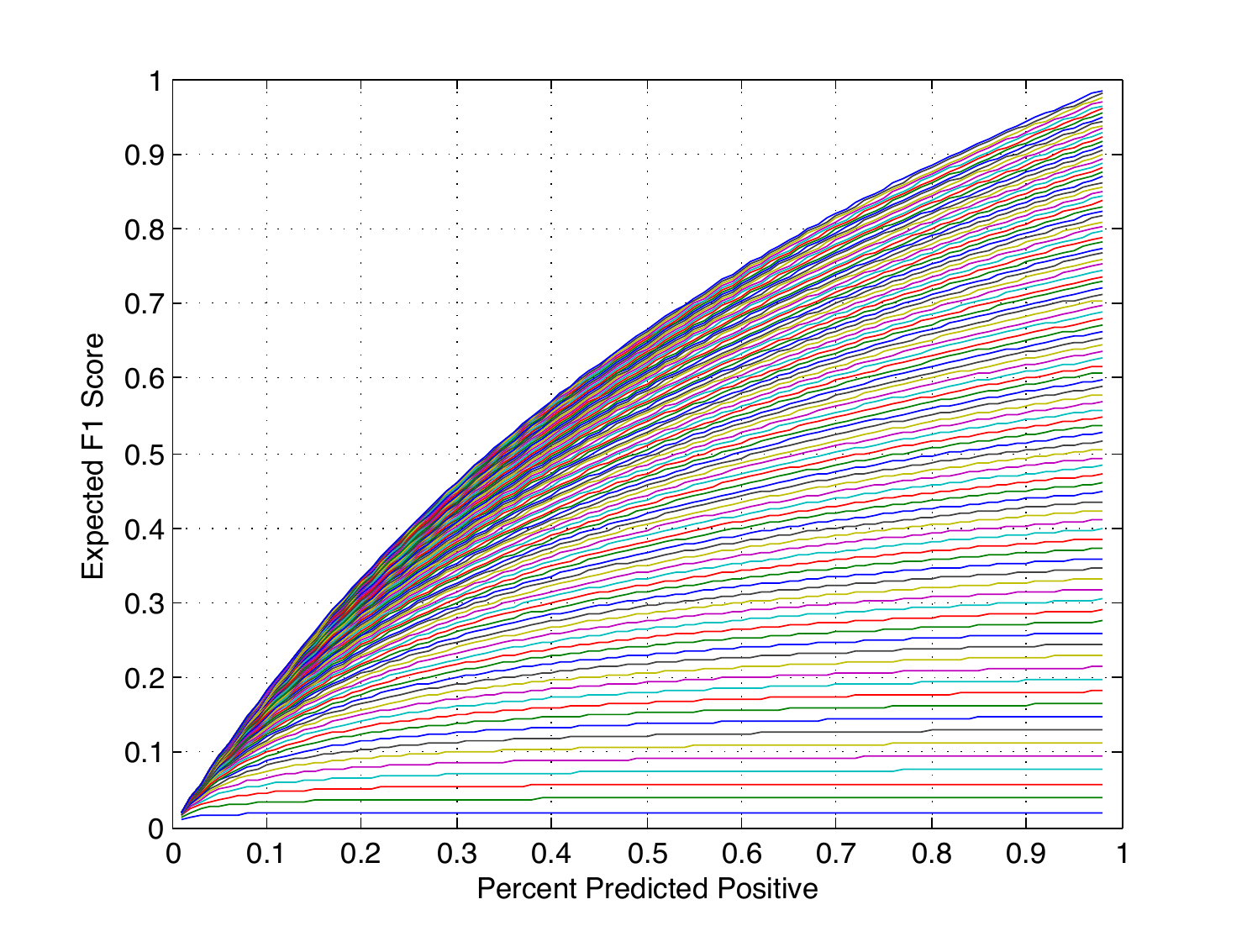}
	\captionof{figure}{The expected F1 score of an optimally thresholded random guess 
	is highly dependent on the base rate. }
	\label{fig:uninformative}
\end{minipage}
\end{figure}
%
\section{Optimal Decision Regions for F1 Maximization}

In this section, we provide a characterization of the optimal decision regions that maximize F1
and, for a special case, 
we present a relationship between the optimal threshold and the maximum achievable F1 score. 

We assume that the classifier outputs real-valued scores $s$ 
and that there exist two distributions $p(s|t=1)$ and $p(s|t=0)$
that are  the conditional probability of seeing the score $s$ 
when the true label $t$ is $1$ or $0$, respectively. 
We assume that these distributions are known in this section;
the next section discusses an empirical version of the result.
Note also that in this section $tp$ etc.~are fractions that sum to one, not counts.

Given $p(s|t=1)$ and $p(s|t=0)$, 
we seek a decision rule $D: s \rightarrow \{0,1\}$ mapping scores to class labels 
such that the resultant classifier maximizes F1. 
We start with a lemma that is valid for any $D$.

\begin{lemma}
The true positive rate $tp = b\int_{s:D(s)=1} p(s|t=1) ds$
where $b = p(t=1)$ is the base rate.
\end{lemma}
\begin{proof}
Clearly $tp= \int_{s:D(s)=1} p(t=1|s) p(s) ds$.
Bayes rule says that $p(t=1|s) = p(s|t=1)p(t=1)/p(s)$.
Hence $tp = b \int_{s:D(s)=1} p(s|t=1) ds$.
\end{proof}
Using three similar lemmas, the entries of the confusion matrix are
\begin{eqnarray*}
tp & = & b\int_{s:D(s)=1} p(s|t=1) ds\\
fn & = & b\int_{s:D(s)=0} p(s|t=1) ds
\end{eqnarray*}
\begin{eqnarray*}
fp & = & (1-b)\int_{s:D(s)=1} p(s|t=0) ds\\
tn & = & (1-b)\int_{s:D(s)=0} p(s|t=0) ds.
\end{eqnarray*}
The following theorem describes the optimal decision rule that maximizes F1.
\begin{theorem}
A score $s$ is assigned to the positive class, that is $D(s)=1$,
by a classifier that maximizes F1 if and only if
\begin{equation}
\frac{b \cdot p(s|t=1)}{(1-b) \cdot p(s|t=0)} \geq J 
\label{eq:jac}
\end{equation}
where $J=\frac{tp}{fn+tp+fp}$ is the Jaccard index of the optimal classifier,
with ambiguity given equality in~(\ref{eq:jac}).
\label{th:maxf1}
\end{theorem}
Before we provide the proof of this theorem,
we note the difference between the rule in (\ref{eq:jac}) 
and conventional cost-sensitive decision making \cite{elkan2001foundations} 
or Neyman-Pearson detection.
In the latter, the right hand side $J$ is replaced by 
a constant $\lambda$ that depends only on the costs of $0-1$ and $1-0$ classification errors,
and {not on the performance of the classifier on the entire batch}. 
We will later elaborate on this point, and describe how this relationship leads to potentially undesirable thresholding behavior 
for many applications in the multilabel setting.  

\begin{proof}
Divide the domain of $s$ into regions of size $\Delta$. 
Suppose that the decision rule $D(\cdot)$ has been fixed for all regions 
except a particular region denoted $\Delta$ around a point (with some abuse of notation) $s$.
Write $P_{1}(\Delta)=\int_{\Delta} p(s|t=1)$ and  define $P_{0}(\Delta)$ similarly.

Suppose that the F1 achieved with decision rule $D$ for all scores besides $D(\Delta)$ 
is  $F1=\frac{2tp}{2tp+fn+fp}$. 
Now, if we add $\Delta$ to the positive part of the decision rule, $D(\Delta)=1$, 
then the new F1 score will be
\[
F1' = \frac{2tp + 2bP_1(\Delta)}{2tp + 2bP_{1}(\Delta) + fn + fp + (1-b)P_0(\Delta)}.
\]
On the other hand, if we add $\Delta$ to the negative part of the decision rule, $D(\Delta)=0$, 
then the new F1 score will be
\[
F1'' = \frac{2tp}{2tp + fn + bP_1(\Delta) + fp}.
\]
We add $\Delta$ to the positive class only if $F1'\geq F1''$. 
With some algebraic simplification, 
this condition becomes
\[
\frac{bP_{1}(\Delta)}{(1-b)P_{0}(\Delta)} \geq \frac{tp}{tp+fn+fp }.
\]
Taking the limit $|\Delta| \rightarrow 0$ gives the claimed result.
\end{proof}
If, as a special case, the model outputs {calibrated} probabilities, 
that is $p(t=1|s)=s$ and $p(t=0|s)=1-s$, 
then we have the following corollary.

\begin{corollary}
An instance with predicted probability $s$ is assigned to the positive class 
by the optimal decision rule that maximizes F1 if and only if $s \geq F/2$
where $F=\frac{2tp}{2tp+fn+fp}$ is the F1 score achieved by this optimal decision rule.
\end{corollary}
\begin{proof}
Using the definition of calibration and then Bayes rule,
for the optimal decision surface for assigning a score $s$ to the positive class
\begin{equation}
\frac{p(t=1|s)}{p(t=0|s)} = \frac{s}{1-s} = \frac{p(s|t=1)b}{p(s|t=0)(1-b)}.
\label{eq:bay}
\end{equation}
Incorporating (\ref{eq:bay}) in Theorem~\ref{th:maxf1} gives
\begin{equation*}
\frac{s}{1-s} \geq \frac{tp}{fn+tp+fp}.
\end{equation*}
Simplifying results in
\begin{equation*}
s \geq \frac{tp}{2tp+fn+fp} = \frac{F}{2}.
\end{equation*} 
\end{proof}
Thus, the optimal threshold in the calibrated case is half the maximum $F1$. 

Above, we assume that scores have a distribution conditioned on the true class. 
Using the intuition in the proof of Theorem~\ref{th:maxf1}, 
we can also derive an empirical version of the result. 
To save space, we provide a more general version of the empirical result in the next section for multilabel problems, noting that a similar non-probabilistic statement holds for the single label setting as well.

\subsection{Maximizing Expected F1 Using a Probabilistic Classifier}

The above result can be extended to the multilabel setting with dependence. 
We give a different proof that confirms the optimal threshold for empirical maximization of F1. 

We first present an algorithm from \cite{dembczynski2011exact}.
Let $\textbf{s}$ be the output vector of length $n$ scores from a model, 
to predict $n$ labels in the multilabel setting. 
Let $\textbf{t}\in\{0,1\}^{n}$ be the gold standard and $\textbf{h}\in\{0,1\}^{n}$ be
the thresholded output for a given set of $n$ labels. 
In addition, define $a=tp+fn$, the total count of positive labels in the gold standard 
and $c=tp +fp$ the total count of predicted positive labels. 
Note that $a$ and $c$ are functions of $\textbf{t}$ and $\textbf{h}$, 
though we suppress this dependence in notation. 
Define $\mathbf{z}^{a}=\sum_{\mathbf{t}:tp+fn=a}\mathbf{t}p(\mathbf{t})$. 
The maximum achievable macro F1 is
\begin{eqnarray*}
F1 &=& \max_{c} \max_{\textbf{h}:tp +fp = c} \mathbbm{E}_{p(\textbf{t}|\textbf{s})} \left[\frac{2tp}{2tp + fp + fn}\right] \\
&=& \max_{c} \max_{\textbf{h}:tp +fp = c} 2\textbf{h}^{T} \sum_{a} \frac{\mathbf{z}^{a}}{a+c} 
\label{eq:f1preconc}.
\end{eqnarray*}

Algorithm: Loop over the number of predicted positives $c$. 
Sort the vector $\sum_{a} \frac{\mathbf{z}^{a}}{a+c}$ of length $n$. 
Proceed along its entries one by one. 
Adding an entry to the positive class increases the numerator by $\mathbf{z}^{a}$, which is always positive. 
Stop after entry number $c$. 
Pick the $c$ value and corresponding threshold which give the largest F1. 

Some algebra gives the following interpretation:
\[ 
\max_{c} \mathbbm{E}(F1) =  \max_{c} \sum_{a}\frac{\mathbbm{E}(tp | c)}{a + c}  p(a).
\]
\begin{theorem}
\label{th:f1multi}
The stopping threshold will be $\max \mathbbm{E}_{p(\textbf{y}|\textbf{s})} [\frac{F1}{2}]$. 
\end{theorem}

%
%
%
%

\subsection{Consequences of F1 Optimal Classifier Design}

We demonstrate two consequences of designing classifiers that maximize F1. 
These are the ``batch observation" and the ``uninformative classifier observation."
We will later demonstrate with a case study that these can combine 
to produce surprising and potentially undesirable optimal predictions when macro F1 is optimized in practice.

The \textbf{batch observation} is that a label may or may not be predicted for an instance 
depending on the distribution of other probabilities in the batch. 
Earlier, we observed a relationship between the optimal threshold and the maximum $\mathbbm{E}(F1)$ 
and demonstrated that the maximum $\mathbbm{E}(F1)$ is related to 
the distribution of probabilities for all predictions. 
Therefore, depending upon the distribution in which an instance is placed, 
it may or may not exceed the optimal threshold. 
Note that because F1 can never exceed 1, the optimal threshold can never exceed .5. 

Consider for example an instance with probability 0.1. 
It will be predicted positive if it has the highest probability of all instances in a batch. 
However, in a different batch, where the probabilities assigned to all other elements are 0.5 and $n$ is large, 
the maximum $\mathbbm{E}(F1)$ would be close to 2/3. 
According to the theorem, we will predict positive on this last instance only if it has a probability greater than 1/3. 

An uninformative classifier is one that assigns the same score to all examples. 
If these scores are calibrated probabilities, the base rate is assigned to every example.  

\begin{theorem}
Given an uninformative classifier for a label, optimal thresholding to maximize F1 results in predicting all examples positive.
\label{th:uninformative}
\end{theorem}
\begin{proof}
Given an uninformative classifier,  we seek the optimal threshold that maximizes $\mathbb{E}(F1)$. 
The only choice is how many labels to predict. 
By symmetry between the instances, it doesn't matter which instances are labeled positive.

Let $a = tp + fn$ be the number of actual positives
and let $c = tp + fp$ be the number of positive predictions.
The denominator of the expression for F1 in Equation~(\ref{eq:f1def}), 
that is $2tp + fp + fn = a + c$, is constant. 
The number of true positives, however, is a random variable.
Its expected value is equal to the sum of the probabilities that 
each example predicted positive actually is positive:
\[ 
\mathbb{E}(F1) = \frac{2\sum_{i=1}^{c} b}{a + c} = \frac{2c \cdot b}{a + c}
\]
where $b = a/n$ is the base rate.
To maximize this expectation as a function of $c$, 
we calculate the partial derivative with respect to $c$, applying the product rule:
\begin{eqnarray} 
\frac {\partial }{\partial c}\mathbb{E}(F1) &=& \frac{\partial}{\partial c}\frac{2c \cdot b}{a + c} 
=  \frac{2b}{a+c}- \frac{2c \cdot b}{(a+c)^2}. \nonumber
\end{eqnarray}
Both terms in the difference are always positive, 
so we can show that this derivative is always positive by showing that
$$
\frac{2b}{a+c} > \frac{2c \cdot b}{(a+c)^2}.
$$
Simplification gives the condition $1 > \frac{c}{a+c}$.
As this condition always holds, the derivative is always positive. 
Therefore, whenever the frequency of actual positives in the test set is nonzero, 
and the classifier is uninformative, expected F1 is maximized by predicting that all examples are positive.
\end{proof}

For low base rates an optimally thresholded uninformative classifier achieves $\mathbbm{E}(F1)$ close to 0, 
while for high base rates $\mathbbm{E}(F1)$ is close to 1 (Figure~\ref{fig:uninformative}). 
We revisit this point in the context of macro F1.

\section{Multilabel Setting} 

Different metrics are used to measure different aspects of a system's performance. 
However, by changing the loss function, this can change the optimal predictions.  
We relate the batch observation to discrepancies between predictions optimal for micro and macro F1. 
We show that while micro F1 is dominated by performance on common labels, 
macro F1 disproportionately weights rare labels.  
Additionally, we show that macro averaging over F1 can conceal uninformative classifier thresholding. 


Consider the equation for F1, and imagine $tp$, $fp$, and $fn$ to be known 
for $m-1$ labels with some distribution of base rates. 
Now consider the $m$th label to be  rare with respect to the distribution. 
A perfect classifier increases $tp$ by a small amount $\varepsilon$ 
equal to the number $b \cdot n$ of actual positives for that rare label, 
while contributing nothing to the counts $fp$ or $fn$:
\[ 
F1' = \frac{2(tp + b \cdot n)}{2(tp + b \cdot n) + fp + fn}.
\]
On the other hand, a trivial prediction of all negative only increases $fn$ by a small amount:
\[ 
F1'' = \frac{2tp}{2tp + fp + (fn+ b \cdot n)} .
 \]
By contrast, predicting all positive for a rare label will increase $fp$ by a large amount $\beta = n - \varepsilon$. 
We have
\[ \frac{F1'}{F1''}=\frac{1+\frac{b \cdot n}{tp}}{1 + \frac{nb}{a+c+b \cdot n}}.
\]
where $a$ and $c$ are the number of positives in the gold standard and the number of positive predictions for the first $m-1$ labels. We have $a+c \leq n \sum_{i} b_{i}$ and so if $b_{m} \ll \sum_{i} b_{i}$ this ratio is small. Thus, performance on rare labels is washed out.

In the single-label setting, the small range between the F1 value achieved by a trivial classifier and a perfect one
may not be problematic. 
If a trivial system gets a score of 0.9, we can adjust the scale for what constitutes a good score. 
However, when averaging separately calculated F1 over all labels, 
this variability can skew scores to disproportionately weight performance on rare labels. 
Consider the two label case when one label has a base rate of 0.5 and the other has a base rate of 0.1. 
The corresponding expected F1 for trivial classifiers are 0.67 and 0.18 respectively. 
Thus the expected F1 for optimally thresholded trivial classifiers is  0.42. 
However, an improvement to perfect predictions on the rare label elevates the macro F1 to 0.84 
while such an improvement on the common label would only correspond to a macro F1 of 0.59. 
Thus the increased variability of F1 results in high weight for rare labels in macro F1. 

For a rare label with an uninformative classifier, micro F1 is optimized by predicting all negative while macro is optimized by predicting all positive. 
Earlier, we proved that the optimal threshold for predictions based on a calibrated probabilistic classifier 
is half of the maximum F1 attainable given any threshold setting. 
In other words, which batch an example is submitted with affects whether a positive prediction will be made. 
In practice, a system may be tasked with predicting labels with widely varying base rates. 
Additionally a classifier's ability to make confident predictions may vary widely from label to label. 

Optimizing micro F1 as compared to macro F1 can be thought of as 
choosing optimal thresholds given very different batches. 
If the base rate and distribution of probabilities assigned to instances vary from label to label, 
so will the predictions. 
Generally, labels with low base rates and less informative classifiers will be over-predicted 
to maximize macro F1 as compared to micro F1. 
We present empirical evidence of this phenomenon in the following case study.

\section{Case Study}

This section discusses a case study that demonstrates how in practice, 
thresholding to maximize macro-F1 can produce undesirable predictions. 
To our knowledge, a similar real-world case of pathological behavior
has not been previously described in the literature,
even though macro averaging F1 is a common approach.

We consider the task of assigning tags from a controlled vocabulary of 26,853 MeSH terms 
to articles in the biomedical literature using only titles and abstracts. 
We represent each abstract as a sparse bag-of-words vector over a vocabulary of 188,923 words. 
The training data consists of a matrix $A$ with $n$ rows and $d$ columns, 
where $n$ is the number of abstracts and $d$ is the number of features in the bag of words representation. 
We apply a tf-idf text preprocessing step to the bag of words representation 
to account for word burstiness \cite{mke05}
and to elevate the impact of rare words.    

Because linear regression models can be trained for multiple labels efficiently, 
we choose linear regression as a model. 
Note that square loss is a proper loss function and does yield calibrated probabilistic predictions
\cite{mjveo12}.
Further, to increase the speed of training and prevent overfitting, 
we approximate the training matrix $A$ by a rank restricted $A_{k}$ using singular value decomposition. 
One potential consequence of this rank restriction is that the signal of extremely rare words can be lost. 
This can be problematic when rare terms are the only features of predictive value for a label.

Given the probabilistic output of the classifier 
and the theory relating optimal thresholds to maximum attainable F1, 
we designed three different plug-in rules to maximize micro, macro and per instance F1. 
Inspection of the predictions to maximize micro F1 revealed no irregularities. 
However, inspecting the predictions thresholded to maximize performance on macro F1
showed that several terms with very low base rates were predicted for more than a third of all test documents.
Among these terms were ``Platypus", ``Penicillanic Acids" and ``Phosphinic Acids" (Figure~\ref{fig:MeSH}). 

\begin{figure}[t]
\centering
	\begin{tabular}{ | l | c | r | r |}
	 \hline
	 MeSH Term & Count & Max F1 & Threshold \\
	  \hline      
 		Humans & 2346 & 0.9160 & 0.458\\
   		Male & 1472 & 0.8055 & 0.403\\
		Female & 1439 & 0.8131& 0.407\\
		\textbf{Phosphinic Acids} & \textbf{1401} & ${1.544 \cdot 10^{-4}}$ & ${7.71 \cdot 10^{-5}}$\\
		\textbf{Penicillanic Acid} & \textbf{1064} & ${8.534 \cdot 10^{-4}}$ & ${4.27 \cdot 10^{-4}}$\\
		Adult & 1063 & 0.7004  & 0.350\\
		Middle Aged & 1028 & 0.7513 & 0.376\\
		\textbf{Platypus} & \textbf{980} & ${4.676 \cdot 10^{-4}}$ & ${2.34 \cdot 10^{-4}}$ \\
	  \hline  
	\end{tabular}
	\caption{Frequently predicted MeSH Terms. When macro F1 is optimized, low thresholds are set for rare labels (bold) with  uninformative classifiers.}
	\label{fig:MeSH}
\end{figure}


In multilabel classification, a label can have low base rate and an uninformative classifier. 
In this case, optimal thresholding requires the system to predict all examples positive for this label. 
In the single-label case, such a system would achieve a low F1 and not be used. 
But in the macro averaging multilabel case, 
the extreme thresholding behavior can take place on a subset of labels, 
while the system manages to perform well overall. 

\section{A Winner's Curse}

In practice, decision regions that maximize F1 are often set experimentally, rather than analytically. 
That is, given a set of training examples, 
their scores and ground truth decision regions for scores that map to different labels 
are set that maximize F1 on the training batch.

In such situations, the optimal threshold can be subject to a winner's curse \cite{capen1971competitive}
where a sub-optimal threshold is chosen because of sampling effects or limited training data. 
As a result, the future performance of a classifier using this threshold is less than the empirical performance. 
We show that threshold optimization for F1 is particularly susceptible to this phenomenon
(which is a type of overfitting).

In particular, different thresholds have different rates of convergence of estimated F1 with number of samples $n$. 
As a result, for a given $n$, comparing the empirical performance of low and high thresholds 
can result in suboptimal performance. 
This is because, for a fixed number of samples, some thresholds converge to their true error rates 
while others have higher variance and may be set erroneously. 
We demonstrate these ideas for a scenario with an uninformative model, though they hold more generally. 

Consider an uninformative model, for a label with base rate $b$. 
The model is uninformative in the sense that output scores are $s_{i} = b+n_{i}\ \forall\ i$, 
where $n_{i}=\mathcal{N}(0,\sigma^{2})$. 
Thus, scores are uncorrelated with and independent of the true labels. 
The empirical accuracy for a threshold $t$ is 
\begin{equation}
A^{t}_{exp} = \frac{1}{n} \sum_{i \in +} \mathbf{1} [S_{i} \geq t] + \frac{1}{n} \sum_{i \in -} \mathbf{1} [S_{i} \leq t] \label{eq:iid}
\end{equation}
where $+$ and $-$ index the positive and negative class respectively. 
Each term in Equation~(\ref{eq:iid}) is the sum of $O(n)$ i.i.d random variables 
and has exponential (in $n$) rate of convergence to the mean 
irrespective of the base rate $b$ and the threshold $t$. 
Thus, for a fixed number $T$ of threshold choices, 
the probability of choosing the wrong threshold $P_{err}\leq T 2^{-\epsilon n}$ 
where $\epsilon$ depends on the distance between the optimal and next nearest threshold. 
Even if errors occur the most likely errors are thresholds close to the true optimal threshold 
(a consequence of Sanov's Theorem \cite{cover2012elements}).

Consider how F1-maximizing thresholds would be set experimentally, 
given a training batch of independent ground truth and scores from an uninformative classifier. 
The scores $s_{i}$ can be sorted in decreasing order (w.l.o.g.) 
since they are independent of the true labels for an uninformative classifier. 
Based on these, we empirically select the threshold that maximizes F1 on the training batch. 
The optimal empirical threshold will lie between two scores that include the value $\frac{F1}{2}$, 
when the scores are calibrated, in accordance with Theorem~\ref{th:maxf1}. 

The threshold $s_{min}$ that classifies all examples positive 
(and maximizes $F1$ analytically by Theorem~\ref{th:uninformative}) 
has an empirical F1 close to its expectation of $\frac{2b}{1+b}=\frac{2}{1 + {1}/{b}}$ 
since $tp$, $fp$ and $fn$ are all estimated from the entire data. 
Consider the threshold $s_{max}$ that classifies only the first example positive and all others negative. 
With probability $b$, this has F1 score ${2}/({2+b \cdot n})$,
which is lower than that of the optimal threshold only when
$$
b \geq  \frac{\sqrt{1+\frac{8}{n}}-1}{2}. 
$$
Despite the threshold $s_{max}$ being far from optimal, 
it has a constant probability of having a higher F1 on training data, 
a probability that does not decrease with $n$, for $n < ({1-b})/{b^2}$. 
Therefore, optimizing F1 will have a sharp threshold behavior, 
where for $n < ({1-b})/{b^2}$ the algorithm will identify large thresholds with constant probability, 
whereas for larger $n$ it will correctly identify small thresholds. 
Note that identifying optimal thresholds for $F1$ is still problematic 
since it then leads to issue identified in the previous section. 
While these issues are distinct, 
they both arise from the nonlinearity of  F1 score and its asymmetric treatment of positive and negative labels.

We simulate this behavior, executing 10,000 runs for each setting of the base rate, 
with $n=10^6$ samples for each run to set the threshold (Figure~\ref{fig:winner}). 
Scores are chosen using variance $\sigma^{2}=1$. 
True labels are assigned at the base rate, independent of the scores. 
The threshold that maximizes F1 on the training set is selected. 
We plot a histogram of the fraction predicted positive as a function of the empirically chosen threshold. 
There is a shift from predicting almost all positives to almost all negatives as base rate is decreased. 
In particular, for low base rate $b$, even with a large number of samples, 
a small fraction of examples are predicted positive. 
The analytically derived optimal decision in all cases is to predict all positive,
i.e.~to use a threshold of 0. 

\begin{figure}[t]
\begin{center}
	\includegraphics[width=4in]{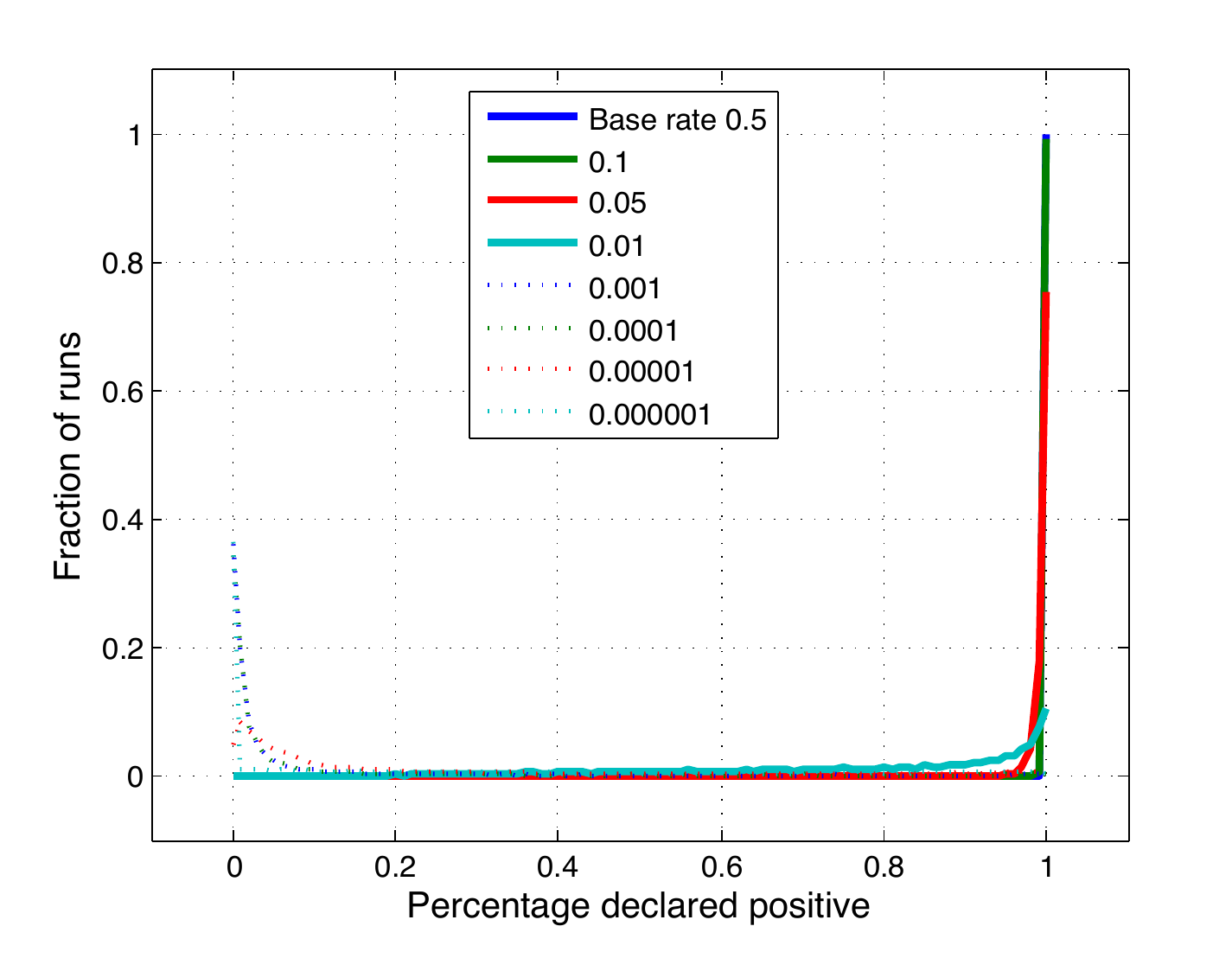}
	\caption{The distribution of experimentally chosen thresholds changes with varying $b$. 
	For small $b$, a small fraction of examples are predicted positive 
	even though the optimal thresholding is to predict all positive.}
	\label{fig:winner}
\end{center}
\end{figure}

\section{Discussion}
In this paper, we present theoretical and empirical results describing the properties of the F1 performance metric for multilabel classification. We relate the best achievable F1 score to the optimal decision-making threshold and show that when a classifier is uninformative, predicting all instances positive maximizes the expectation of F1. Further, we show that in the multilabel setting, this behavior can be problematic when the metric to maximize is macro F1 and for a subset of rare labels the classifier is uninformative. In contrast, we demonstrate that given the same scenario, expected micro F1 is maximized by predicting all examples to be negative. This knowledge can be useful as such scenarios are likely to occur in settings with a large number of labels. We also demonstrate that micro F1 has the potentially undesirable property of washing out performance on rare labels. 

No single performance metric can capture every desirable property. For example, separately reporting precision and recall is more informative than reporting F1 alone. Sometimes, however, it is practically necessary to define a single performance metric to optimize. Evaluating competing systems and objectively choosing a winner presents such a scenario. In these cases, a change of performance metric can have the consequence of altering optimal thresholding behavior.

%

\bibliography{sources}

\begin{thebibliography}{10}
\providecommand{\url}[1]{\texttt{#1}}
\providecommand{\urlprefix}{URL }

\bibitem{akay2009support}
Akay, M.F.: Support vector machines combined with feature selection for breast
  cancer diagnosis. Expert Systems with Applications  36(2),  3240--3247 (2009)

\bibitem{capen1971competitive}
Capen, E.C., Clapp, R.V., Campbell, W.M.: Competitive bidding in high-risk
  situations. Journal of Petroleum Technology  23(6),  641--653 (1971)

\bibitem{cover2012elements}
Cover, T.M., Thomas, J.A.: Elements of information theory. John Wiley \& Sons
  (2012)

\bibitem{jose2009learning}
del Coz, J.J., Diez, J., Bahamonde, A.: Learning nondeterministic classifiers.
  Journal of Machine Learning Research  10,  2273--2293 (2009)

\bibitem{plugin}
Dembczynski, K., Kot{\l}owski, W., Jachnik, A., Waegeman, W., H{\"u}llermeier,
  E.: Optimizing the {F}-measure in multi-label classification: Plug-in rule
  approach versus structured loss minimization. In: ICML (2013)

\bibitem{dembczynski2011exact}
Dembczy{\'n}ski, K., Waegeman, W., Cheng, W., H{\"u}llermeier, E.: An exact
  algorithm for {F}-measure maximization. In: Neural Information Processing
  Systems (2011)

\bibitem{elkan2001foundations}
Elkan, C.: The foundations of cost-sensitive learning. In: International joint
  conference on artificial intelligence. pp. 973--978 (2001)

\bibitem{jansche2007maximum}
Jansche, M.: A maximum expected utility framework for binary sequence labeling.
  In: Annual Meeting of the Association For Computational Linguistics. p. 736
  (2007)

\bibitem{lewis1995}
Lewis, D.D.: Evaluating and optimizing autonomous text classification systems.
  In: Proceedings of the 18th annual international ACM SIGIR conference on
  research and development in information retrieval. pp. 246--254. ACM (1995)

\bibitem{mke05}
Madsen, R., Kauchak, D., Elkan, C.: Modeling word burstiness using the
  {Dirichlet} distribution. In: Proceedings of the International Conference on
  Machine Learning (ICML). pp. 545--552 (Aug 2005)

\bibitem{manning2008introduction}
Manning, C., Raghavan, P., Sch{\"u}tze, H.: Introduction to information
  retrieval, vol.~1. Cambridge University Press (2008)

\bibitem{mjveo12}
Menon, A., Jiang, X., Vembu, S., Elkan, C., Ohno-Machado, L.: Predicting
  accurate probabilities with a ranking loss. In: Proceedings of the
  International Conference on Machine Learning (ICML) (Jun 2012)

\bibitem{mozer2001prodding}
Mozer, M.C., Dodier, R.H., Colagrosso, M.D., Guerra-Salcedo, C., Wolniewicz,
  R.H.: Prodding the {ROC} curve: Constrained optimization of classifier
  performance. In: NIPS. pp. 1409--1415 (2001)

\bibitem{musicant2003optimizing}
Musicant, D.R., Kumar, V., Ozgur, A., et~al.: Optimizing {F}-measure with
  support vector machines. In: FLAIRS Conference. pp. 356--360 (2003)

\bibitem{systematic}
Sokolova, M., Lapalme, G.: A systematic analysis of performance measures for
  classification tasks. Information Processing and Management  45,  427--437
  (2009)

\bibitem{suzuki2006training}
Suzuki, J., McDermott, E., Isozaki, H.: Training conditional random fields with
  multivariate evaluation measures. In: Proceedings of the 21st International
  Conference on Computational Linguistics and the 44th annual meeting of the
  Association for Computational Linguistics. pp. 217--224. Association for
  Computational Linguistics (2006)

\bibitem{tanmacro}
Tan, S.: Neighbor-weighted k-nearest neighbor for unbalanced text corpus.
  Expert Systems with Applications  28,  667--671 (2005)

\bibitem{tsoumakas2007multi}
Tsoumakas, Grigorios \&~Katakis, I.: Multi-label classification: An overview.
  International Journal of Data Warehousing and Mining  3(3),  1--13 (2007)

\bibitem{ye2012optimizing}
Ye, N., Chai, K.M., Lee, W.S., Chieu, H.L.: Optimizing {F}-measures: A tale of
  two approaches. In: Proceedings of the International Conference on Machine
  Learning (2012)

\bibitem{zhao2013beyond}
Zhao, M.J., Edakunni, N., Pocock, A., Brown, G.: Beyond {Fano's} inequality:
  Bounds on the optimal {F}-score, {BER}, and cost-sensitive risk and their
  implications. Journal of Machine Learning Research  14(1),  1033--1090 (2013)

\end{thebibliography}
\bibliographystyle{splncs03}

\end{document}